\documentclass[journal]{IEEEtran}

\usepackage[utf8]{inputenc}
\usepackage[T1]{fontenc}
\usepackage{amsmath,amssymb,amsfonts}
\usepackage{amsthm}
\usepackage{graphicx}
\usepackage{booktabs}
\usepackage{array}
\usepackage{cite}
\usepackage{enumitem}
\usepackage[colorlinks=true,linkcolor=blue,citecolor=blue,urlcolor=blue]{hyperref}

\newtheorem{theorem}{Theorem}
\newtheorem{proposition}{Proposition}
\newtheorem{definition}{Definition}

\begin{document}

\title{Behavior--Realization Separation for Constrained Physical Human--Robot Interaction}

\author{Yongyan~Cao%
\thanks{Y. Cao is with Voryx Robotics LLC, San Jose, CA 95136, USA
(e-mail: yongyancao@gmail.com).}}

\maketitle

\begin{abstract}
Physical human--robot interaction software often couples desired-behavior
specification with constrained realization; we treat these as separate
layers. A \emph{behavior layer} supplies a desired contact-port
acceleration $a_k^{\mathrm{id}}=f_\theta(e_k,\dot e_k,F_{h,k})$. A
\emph{realization layer} converts it into constrained robot commands and
reports total desired-versus-realized acceleration error instead of hiding
it in saturation. A same-objective unconstrained counterfactual separates
regularization from constraint intervention, while plant data expose model
error. This paper implements a receding-horizon
quadratic program realizing memoryless affine behaviors. Changing the
behavior modifies objective coefficients through $(C_\theta,G_\theta)$
while the robot-command variable and feasible set remain unchanged. A
planar study instantiates impedance and admittance; the same running
layer accepts an impedance--admittance--impedance reassignment without
reconstruction, under its existing rate limit. On a torque-controlled
7-DOF Franka FR3 in MuJoCo, the runtime freezes task-space dynamics per
solve and enforces torque feasibility across its horizon. Under a
sustained 20~N push, it holds a slack-relaxed workspace boundary to
within approximately 0.1--0.2~mm, versus 4.4~cm (impedance) and 4.7~cm
(admittance) overshoot from instantaneous clipping. A derated actuator
budget then activates the torque constraint: horizon-wide enforcement
keeps its frozen-model plan feasible to $2.1\times10^{-4}$~N$\cdot$m,
whereas a first-step-only ablation plans up to 11.329~N$\cdot$m beyond budget; on
the executed nonlinear plant, where both share the same local-model
error, the gap is smaller but still favors horizon-wide enforcement
(0.161 vs.\ 0.380~N$\cdot$m). These results are a focused proof of
behavior--realization separation, not a new MPC or safety-filter principle
(scope in Section~\ref{sec:scope}).
\end{abstract}


\section{Introduction}
\label{sec:intro}

Interaction behavior should specify only desired interaction dynamics;
physical feasibility should be realized independently by the robot.

A pHRI software stack has two conceptually different responsibilities:
\begin{enumerate}
\item specify the desired interaction behavior; and
\item realize that behavior through commands compatible with the physical
robot.
\end{enumerate}

\noindent\textbf{Architectural hypothesis.} We argue for a general
software-engineering principle, stated above, and for a specific reason
it should hold: desired interaction behavior is a statement in
interaction-state space --- position error, velocity error, and human
force --- while physical feasibility is a statement in robot-constraint
space --- joint torque, joint limits, and workspace geometry. These are
different mathematical objects with different natural variables, so a
representation that mixes them (as a fixed impedance or PD law does)
couples two things that do not have to be coupled. Separating them into a
behavior layer and a realization layer connected by an explicit interface
is therefore not only a convenient software boundary but a decomposition
along the actual variables each responsibility depends on. This is an
architectural claim, not merely an empirical one: it also implies that a
behavior specification can be authored, tested, and replaced
independently of the robot, and that feasibility logic can be verified
once and reused across every behavior that respects the interface
(Section~\ref{sec:why} argues this directly). This paper tests the
principle with one executable desired-acceleration interface, developed
and validated only for its memoryless affine subclass
(Sections~\ref{sec:generators}--\ref{sec:fr3}); that affine QP realization
layer is one instance supporting the general claim, not the claim itself,
and its success does not establish that every possible behavior
representation can use the same runtime unchanged.

Conventional impedance control combines behavior specification and
command realization in one mass--spring--damper feedback law.
Variable-impedance methods add adaptation, and model-predictive variants
can optimize impedance parameters or jointly plan motion and compliance.
Those formulations are useful, but changing the behavior representation
generally changes the controller that interprets it.

A concrete consequence of that coupling is actuator saturation. A
conventional impedance or PD law computes joint torque directly from the
interaction error, $\tau=J^\top(-K_de-D_d\dot e)+\tau_{\mathrm{ff}}$,
where $\tau_{\mathrm{ff}}$ includes gravity/Coriolis compensation and, on
a redundant arm, orientation-hold and null-space terms. Neither term
reasons about the actuator limit. A stiff impedance, large interaction
force, or unfavorable configuration can therefore drive the command past
$\tau_{\max}$, and clipping silently changes the delivered acceleration.

The architecture studied here assigns this conflict to a realization
layer. Its current predictive implementation constrains the \emph{total
frozen-model torque} --- feedforward, orientation, null-space, and
interaction terms --- at every prediction step
(Section~\ref{sec:fr3-architecture}). Any deviation from the behavior
specification appears in the total realization residual; a counterfactual
audit distinguishes constraint intervention from secondary-objective and
model terms. Because the
manipulator prediction is frozen at each solve, this remains a
local-model mechanism rather than a global nonlinear-plant guarantee.

We study a different decomposition:

\begin{figure}[t]
\centering
\includegraphics[width=0.95\linewidth]{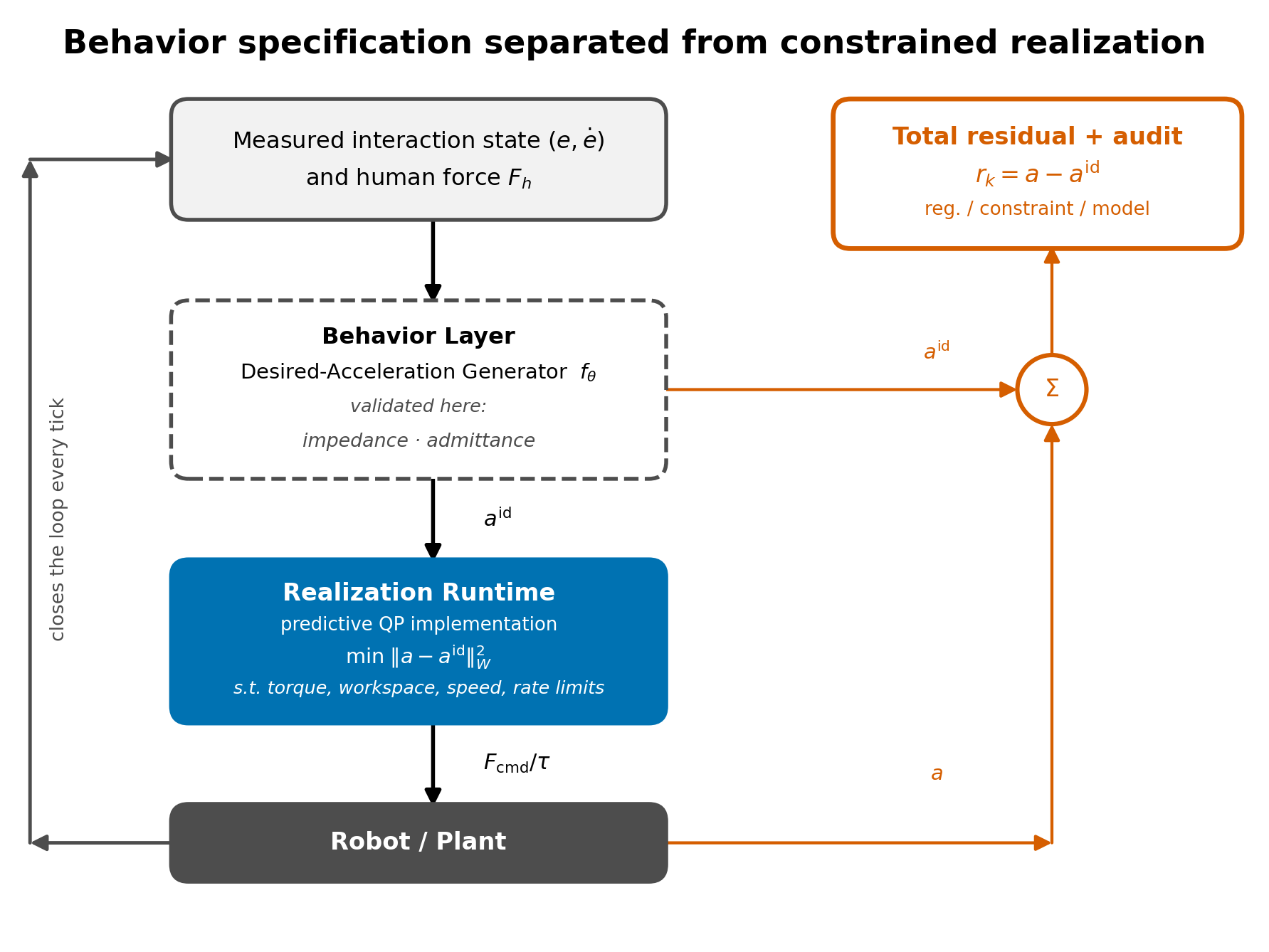}
\caption{Behavior and realization as separate layers. The behavior
layer (dashed) supplies $a^{\mathrm{id}}$; the realization runtime
(blue) converts it into a constrained command and reports
$r_k=a-a^{\mathrm{id}}$; the audit in
(\ref{eq:residual-decomposition}) separates its causes.}
\label{fig:architecture}
\end{figure}

The behavior layer expresses intent without choosing a robot command. The
realization runtime handles feasibility, constraint enforcement, and
command optimization without deciding what stiffness, damping, or
force-response semantics should mean. The runtime reports total behavior
error and a constraint-intervention diagnostic rather than conflating the
two.

\noindent\textbf{The current executable instance.} For memoryless affine
behavior models, the predictive runtime retains one decision variable,
feasible set, and objective template; the desired behavioral acceleration
enters through $(C_\theta,G_\theta)$. Section~\ref{sec:properties} states
this limited structural property formally, and
Section~\ref{sec:generators} implements impedance and admittance.
Nonlinear or stateful specifications may require different prediction
machinery and are not instances of the present QP by assertion.

This separation changes the scientific question. Instead of asking
whether a particular MPC formulation produces compliant tracking, we ask:

\begin{quote}
\emph{How should robot-control software separate desired interaction
behavior from its constrained physical realization, and how should
unavoidable intervention be exposed?}
\end{quote}

The contributions are:
\begin{itemize}
\item a behavior--realization architecture that assigns behavior
semantics and physical command feasibility to separate software layers;
\item a precise desired-acceleration interface, a total realization
residual, and a same-objective constrained/unconstrained audit that makes
constraint intervention separately observable;
\item one predictive realization runtime for the memoryless affine
subclass, with a fixed robot-command variable and feasible set across
impedance and admittance behavior layers;
\item reproducible planar and FR3 simulations showing online
behavior-layer replacement, anticipatory state-constraint handling, and
manipulator-level realization under command and torque limits.
\end{itemize}

We do \textbf{not} claim that architectural separation in the abstract,
impedance behavior in MPC, or constrained interaction control is new.
Task-space inverse dynamics, whole-body QPs, reference governors, and
hierarchical robot controllers already separate related responsibilities,
while model predictive impedance and
interaction controllers provide important integrated formulations. The
narrower contribution is the four-part combination of (i) an interaction
law evaluated along predicted states and exchanged at an acceleration-level
contract, (ii) one fixed robot-command runtime across generators, (iii) a
same-objective constrained/unconstrained counterfactual that decomposes the
realization residual, and (iv) live generator substitution without rebuilding
runtime state. The executable evidence covers two affine behaviors.

\section{Related Work and Positioning}
\label{sec:related}

The closest literature is best distinguished by what the predictive or
adaptive layer optimizes and by what quantity the robot-level controller
is asked to reproduce.

\begin{table*}[t]
\centering
\caption{Positioning against related work.}
\label{tab:related}
\small
\begin{tabular}{@{}>{\raggedright\arraybackslash}p{0.19\linewidth}
>{\raggedright\arraybackslash}p{0.27\linewidth}
>{\raggedright\arraybackslash}p{0.24\linewidth}
>{\raggedright\arraybackslash}p{0.16\linewidth}@{}}
\toprule
Class & Behavior representation & Optimized robot-level variable & Role of prediction \\
\midrule
Predictive variable impedance & $M_d,D_d,K_d$, often with a trajectory & Impedance parameters and/or trajectory & Select compliant behavior \\
Reference-model adaptive impedance & $M_d\ddot e+D_d\dot e+K_de=F_h$ & Adaptive feedback parameters & Usually absent \\
Direct robot MPC & State or tracking-error dynamics & Force, torque, velocity, or acceleration & Optimize robot motion directly \\
Impedance/interaction MPC & Impedance, force, or coupled interaction model & Robot command and sometimes behavior parameters & Embed interaction objectives and constraints \\
Reference governor & A single reference/setpoint signal & Filtered reference fed to a fixed inner-loop controller & Supervise a non-predictive inner loop to enforce constraints \\
Predictive safety filter & Nominal controller input & Minimally modified admissible input & Certify or replace a proposed command \\
Constrained compliance QP & Desired impedance/compliance response & Torque, velocity, or acceleration & Reproduce compliant behavior under robot constraints \\
TSID / hierarchical whole-body QP & Desired task accelerations and task priorities & Joint acceleration, contact force, and/or torque & Usually instantaneous constrained realization \\
Predictive hierarchical task-space control & Task trajectories and acceleration-level errors & Robot state and command trajectories & Extend hierarchical task objectives over a horizon \\
This work & $a^{\mathrm{id}}=C_\theta x+G_\theta F_h$ & Robot command $u$ & Minimize dynamics-realization error under constraints \\
\bottomrule
\end{tabular}
\end{table*}

Several approaches use MPC to adapt impedance parameters or to plan
impedance and motion together. Anand et al.\ place a learned predictive
policy above a low-level variable-impedance controller~\cite{ref1}.
Haninger et al.\ predict interaction and optimize trajectory and
impedance online while incorporating safety constraints~\cite{ref2}.
Recent predictive variable-impedance formulations continue this
parameter-adaptation direction~\cite{ref3}. Their optimized behavior
variables are stiffness, damping, or related trajectory parameters.

A related line of work is model-reference adaptive impedance control,
which specifies desired impedance dynamics and designs an adaptive
nonlinear controller so that the robot approaches them~\cite{ref4}. This
establishes the value of separating a desired model from the physical
robot, but it does not by itself provide horizon-wide handling of coupled
state and input constraints.

A separate class applies MPC directly to the robot or to a
feedback-linearized error model, with decision variables that are robot
commands rather than impedance parameters. Such a controller can generate
an effective closed-loop impedance, but unless a desired interaction
model appears explicitly in the prediction objective, that impedance is
an induced property of the optimizer rather than the behavior it is
asked to realize. This distinction separates the present formulation
from our earlier double-integrator tracking MPC: the earlier controller
penalizes interaction error, whereas the present controller penalizes
disagreement with an independently specified interaction dynamical
system.

Task-space inverse dynamics and hierarchical whole-body QPs are a more
fundamental precedent for the interface itself: they already accept desired
task accelerations as modular objectives and solve for dynamically feasible
robot accelerations, contact forces, or torques. Del Prete et al.'s TSID, for
example, treats prioritized motion/force tasks for constrained fully actuated
robots~\cite{ref19}. Predictive hierarchical task-space control extends
acceleration-error objectives over a finite horizon for underactuated and
constrained robots~\cite{ref20}. Thus neither ``desired task acceleration''
nor ``QP realization under robot constraints'' is claimed as new. Relative
to these frameworks, the contribution tested here is narrower: the task is
an interaction dynamical law evaluated along the predicted state; the same
runtime is retained across impedance and admittance generators; its residual
is split using a same-objective counterfactual; and the generator is replaced
online without reconstructing the runtime state.

Closer still is model predictive impedance and interaction control.
Bednarczyk et al.\ formulate impedance behavior through MPC and handle
practical constraints such as velocity, energy, and jerk~\cite{ref5}.
Gold et al.\ formulate model predictive interaction control using robot
and interaction models in the optimal-control problem and express
manipulation objectives through costs and constraints~\cite{ref6}.
Minniti et al.\ combine whole-body MPC with online adaptation of
robot--environment interaction models for constrained mobile
manipulation~\cite{ref7}. These papers are close precedents and make a
broad ``first constrained interaction-dynamics MPC'' claim untenable.
Alharbat et al.\ explicitly compare NMPC impedance, cascaded admittance,
and hybrid position/force paradigms~\cite{ref14}; their later predictive
admittance controller predicts both robot and desired impedance states,
respects actuator constraints, and is validated in physical aerial
interaction~\cite{ref15}. Kumbhar et al.\ pass admittance-generated plans
to a constrained whole-body QP on Digit hardware~\cite{ref17}. These are
stronger validation precedents than the present simulation-only study;
the remaining distinction sought here is the fixed cross-behavior
desired-acceleration contract, not predictive compliance itself.

QP realization of compliant behavior and residual reporting also predate
this work. Leboutet et al.\ compute robot commands under kinematic and
dynamic constraints and propagate generalized-force residuals when a
limb cannot realize the requested compliance~\cite{ref16}. More recent
safety-critical adaptive impedance work filters a nominal impedance law
through a QP with joint-state and torque constraints and provides formal
invariance and boundedness results~\cite{ref18}. Accordingly, neither
``QP-constrained compliance'' nor ``an explicit residual'' is claimed
independently as novel here.

Reference governors and predictive safety filters are the closest
architectural relatives to this paper's separation, from a different literature than the
parameter-adaptation and direct-MPC work described above. A reference
governor sits upstream of a fixed, already-designed inner-loop controller
and modifies the reference signal so that constraints on the resulting
closed-loop trajectory are respected~\cite{ref8}. The parallel is direct:
a reference governor also separates what is commanded from what is safe
to execute. The difference is where prediction lives and what is
exchanged. A reference governor typically supervises a fixed inner loop
and manipulates a reference signal; the realization layer here is itself
predictive, and the exchanged object is an affine desired-acceleration
law evaluated over predicted interaction states. The realization residual
additionally measures the difference between requested and delivered
dynamics. Predictive safety filters similarly accept an arbitrary
upstream command and use model-based backup optimization to certify or
modify it~\cite{ref13}. This connection limits the novelty claim: the contribution is
not separation in the abstract, but its formulation and measurement at
the interaction-dynamics level.

Control barrier function (CBF) quadratic programs are another related
architectural pattern from a different literature: a CBF-QP takes a
nominal control input and projects it, at each tick, onto the
minimal-deviation admissible set defined by a safety
constraint~\cite{ref10}. The realization layer here is structurally
similar --- a minimal-deviation projection of a requested behavior onto
a robot-feasible set --- but the projected object is a full behavior
law evaluated over a receding horizon rather than a single nominal
input. The total deviation is reported, and the constrained/unconstrained
counterfactual identifies its constraint-intervention component rather
than treating every nonzero residual as safety action. We do
not compare against a CBF-QP filter or a reference-governor instance in
this paper; Section~\ref{sec:scope} lists a matched comparison against
the closest alternative architecture as open.

Theorem~\ref{thm:invariance} is a structural statement about a
parametric quadratic program: for the affine generator class, the QP's
decision variable, feasible set, and constraint structure are fixed
while only its cost coefficients vary with the generator parameter
$\theta$. This is the object studied by the explicit MPC / multi-parametric
QP literature~\cite{ref11,ref12}, which characterizes the solution map
$\theta\mapsto U_k^\star(\theta)$ as piecewise affine. The present paper
does not use that machinery --- the QP is solved online at every tick,
not precomputed offline --- but Theorem~\ref{thm:invariance} is best
read as an instance of that established parametric-QP structure applied
to the behavior--realization interface, not as a new structural fact
about quadratic programs.

Most rows above optimize a quantity tied to one
behavior representation: impedance parameters, adaptive impedance gains,
or robot-level tracking against a specific interaction model. The
reference-governor and safety-filter rows are architecturally closest. Our narrower
distinction is an interaction-level behavior contract whose desired
acceleration is separate from robot feasibility, together with a total
residual and counterfactual intervention audit. The executable and theoretical
evidence is limited to memoryless affine impedance and admittance
behaviors.

\section{Behavior--Realization Formulation}
\label{sec:behaviors}

\subsection{Behavior--Realization Interface}
\label{sec:interface}

Consider a rigid manipulator
\begin{equation}
M(q)\ddot q+C(q,\dot q)\dot q+g(q) =\tau+J(q)^\top F_h,
\end{equation}
where $q\in\mathbb R^{n_q}$, $\tau\in\mathbb R^{n_q}$, and $F_h$ is the
human wrench applied at the interaction port. For a task coordinate
$y=h(q)\in\mathbb R^{n_y}$,
\begin{equation}
\dot y=J(q)\dot q, \qquad \ddot y=J(q)\ddot q+\dot J(q,\dot q)\dot q.
\end{equation}
Substitution of the joint dynamics gives the affine acceleration map
\begin{equation}
\ddot y =b_y(q,\dot q,F_h)+G_y(q)\tau,
\end{equation}
where
\begin{align}
b_y(q,\dot q,F_h) &= J M^{-1}\left(-C\dot q-g+J^\top F_h\right)+\dot J\dot q,\\
G_y(q)&=JM^{-1}.
\end{align}
Let $y_d$ be a nominal interaction pose and define
\begin{equation}
e=y-y_d, \qquad \dot e=\dot y-\dot y_d.
\end{equation}

\begin{definition}[Behavior-layer interface]
\label{def:interface}
In the present architecture, the behavior layer is a causal map
\begin{equation}
\ddot e^{\mathrm{id}} =f_\theta(e,\dot e,F_h,z),
\end{equation}
where $z$ denotes optional internal behavior state and $\theta$ denotes
behavior parameters. Its output is a desired contact-port acceleration
--- termed the \emph{desired behavioral acceleration} throughout this
paper --- not a robot command. The behavior layer is deliberately
unaware of robot torque limits, joint limits, or workspace geometry;
those belong to the realization layer. The general map defines the
interface boundary; the theorem and experiments below cover only its
memoryless affine specialization.
\end{definition}

The architecture does not ask the behavior layer to compute $\tau$. It
supplies $\ddot e^{\mathrm{id}}$, and the realization layer selects
$\tau$ so that the physical $\ddot e$ approaches $\ddot e^{\mathrm{id}}$
while satisfying robot constraints. The predictive QP developed below is
one realization-layer implementation, not the definition of the layer
itself.

The full manipulator formulation is the target architecture. The planar
study below uses the following exactly discretized specialization so
that behavior--realization separation can be isolated without kinematic
or inverse-dynamics confounds: a planar point mass,
\begin{equation}
m_r\ddot p(t)=u(t)+F_h(t),
\end{equation}
where $p\in\mathbb R^2$ is displacement from a nominal interaction pose,
$u\in\mathbb R^2$ is commanded robot force, $F_h\in\mathbb R^2$ is
measured human force, and $m_r>0$ is the realized robot mass. Define
\begin{equation}
x=\begin{bmatrix}p^\top&v^\top\end{bmatrix}^\top,\qquad v=\dot p.
\end{equation}
Under zero-order hold with period $\Delta t$,
\begin{equation}
x_{k+1}=Ax_k+B(u_k+F_{h,k}),
\end{equation}
\begin{equation}
A= \begin{bmatrix} I&\Delta tI\\0&I \end{bmatrix}, \qquad
B= \begin{bmatrix} \frac{\Delta t^2}{2m_r}I\\ \frac{\Delta t}{m_r}I \end{bmatrix}.
\end{equation}
The actual acceleration is
\begin{equation}
a_k=\frac{u_k+F_{h,k}}{m_r}.
\end{equation}

For this planar plant, Definition~\ref{def:interface} specializes with
$x_k=[p_k^\top\ v_k^\top]^\top$ in place of $(e,\dot e)$:
\begin{equation}
a_k^{\mathrm{id}} =f_\theta(x_k,F_{h,k},z_k).
\end{equation}
For the affine class considered throughout this paper,
\begin{equation}
a_k^{\mathrm{id}}=C_\theta x_k+G_\theta F_{h,k}.
\end{equation}

The gap between what a behavior specifies and what the robot delivers is
the \emph{total realization residual},
\begin{equation}
r_k=a_k-a_k^{\mathrm{id}}.
\end{equation}
If $r_k=0$, the robot exactly realizes the desired behavioral
acceleration at that sample. Nonzero $r_k$ alone, however, does not prove
that a constraint caused the deviation: the finite force and force-rate
weights can already shift the unconstrained optimum, and a physical plant
can differ from the prediction model.

We therefore audit three terms. Let $U_k^{c\star}$ be the deployed
constrained optimum and $U_k^{0\star}$ the counterfactual optimum of the
\emph{same} horizon objective at the same state, force forecast, and
previous command after removing physical constraints. Let $a_k^c$ and
$a_k^0$ be their first-step model accelerations, and let $a_k^{\rm emp}$
be the acceleration measured from the plant. Define
\begin{align}
r_k^{\rm reg}&=a_k^0-a_k^{\rm id}, &
r_k^{\rm con}&=a_k^c-a_k^0,\\
r_k^{\rm mod}&=a_k^{\rm emp}-a_k^c. &&
\end{align}
Then the empirical total error closes exactly as
\begin{equation}
r_k^{\rm emp}=r_k^{\rm reg}+r_k^{\rm con}+r_k^{\rm mod}.
\label{eq:residual-decomposition}
\end{equation}
Here $r^{\rm reg}$ contains secondary-objective compromise,
$r^{\rm con}$ is the joint intervention of all enforced constraints, and
$r^{\rm mod}$ contains frozen-model and acceleration-measurement error.
The split does not attribute $r^{\rm con}$ to one individual active row;
that would require a multiplier- or leave-one-constraint-out analysis.
For the exact planar model $r^{\rm mod}=0$. The FR3 audit evaluates the
first two terms at 50~Hz manager updates and obtains $r^{\rm mod}$ from
the following 1~kHz plant sample.

All four quantities retain physical acceleration units. They are logging
and causal-diagnostic signals, not normalized cross-generator scores; a
generator-relative or energy-normalized metric would answer a different
comparison question.

This definition is intentionally stronger than commanding a nominal
impedance force and penalizing
\begin{equation}
\left\|u_k-u_k^{\mathrm{imp}}\right\|_2^2.
\end{equation}
The latter is controller-output tracking: it penalizes deviation from a
nominal force command. $\|r_k\|_W^2$ is tracking too, but of a
categorically different target --- the \emph{desired behavioral
acceleration} the generator specifies, not a prior controller's output
--- so it compares two physical accelerations directly. In this paper
that distinction is demonstrated for impedance and admittance, neither of
which requires the QP to track a prior controller output.

\subsection{Implemented Behavior Layers}
\label{sec:generators}

The two behavior layers implemented below use the same QP template
developed in Section~\ref{sec:qp} and differ only in $(C_\theta,G_\theta)$,
as formalized by Theorem~\ref{thm:invariance}. The first is an impedance
generator, whose desired impedance is
\begin{equation}
M_d a^{\mathrm{id}}+D_dv+K_dp=F_h.
\end{equation}
For isotropic parameters,
\begin{equation}
a^{\mathrm{id}} =-\frac{K_d}{M_d}p-\frac{D_d}{M_d}v+\frac{1}{M_d}F_h,
\end{equation}
so
\begin{equation}
C_{\mathrm{imp}} =\begin{bmatrix} -K_dM_d^{-1}I&-D_dM_d^{-1}I \end{bmatrix},
\qquad G_{\mathrm{imp}}=M_d^{-1}I.
\end{equation}
Unlike variable-impedance MPC (Section~\ref{sec:related}), $M_d,D_d,K_d$
are not decision variables in the present formulation; they describe the
requested behavior.

The second is a force-guided admittance generator, which requests a
force-dependent velocity:
\begin{equation}
T_a a^{\mathrm{id}}+v=YF_h,
\end{equation}
or
\begin{equation}
a^{\mathrm{id}}=-T_a^{-1}v+T_a^{-1}YF_h.
\end{equation}
It has no position-restoring term. After force release, velocity decays
and the displaced position is retained. This provides a qualitatively
different behavior through the same acceleration interface.

\subsection{Predictive Realization Runtime}
\label{sec:qp}

Section~\ref{sec:generators} supplies $a^{\mathrm{id}}$. For a memoryless
affine generator it is represented by $C_\theta,G_\theta$, producing one
condensed-QP template whose numerical cost coefficients change with the
generator while its decision variable and feasible set do not. At time
$k$, let
\begin{equation}
U_k= \begin{bmatrix} u_{0|k}^\top&\cdots&u_{N-1|k}^\top \end{bmatrix}^\top.
\end{equation}
The controller solves
\begin{equation}
\begin{aligned}
\min_{U_k}\quad& \sum_{i=0}^{N-1}\Big(
\|a_{i|k}-a_{i|k}^{\mathrm{id}}\|_W^2
+\lambda_u\|u_{i|k}\|_2^2\\
&\hspace{25mm}+\lambda_{\Delta u}
\|u_{i|k}-u_{i-1|k}\|_2^2\Big)\\
\text{s.t.}\quad&x_{i+1|k}=Ax_{i|k}+B(u_{i|k}+\hat F_{h,i|k}),\\
&a_{i|k}=(u_{i|k}+\hat F_{h,i|k})/m_r,\\
&a_{i|k}^{\mathrm{id}} =C_\theta x_{i|k}+G_\theta\hat F_{h,i|k},\\
&|u_{i|k}|_\infty\le u_{\max},\\
&|u_{i|k}-u_{i-1|k}|_\infty \le\dot u_{\max}\Delta t,\\
&|p_{i+1|k}|_\infty\le p_{\max},\\
&|v_{i+1|k}|_\infty\le v_{\max}.
\end{aligned}
\end{equation}
Only $u_{0|k}^\star$ is executed. $u_{-1|k}$ in the $i=0$ rate penalty
and rate constraint denotes the command actually applied at the previous
solve, carried by the controller between calls (zero before the first
solve), so both are well-defined from the very first call. The
implementation uses the measured force held constant across the horizon,
$\hat F_{h,i|k}=F_{h,k}$, rather than oracle knowledge of the scripted
force. A learned or physics-based force predictor can be substituted
without changing the generator interface.

The objective $\|a_{i|k}-a_{i|k}^{\mathrm{id}}\|_W^2$ is tracking --- of a
categorically different target than a position or velocity trajectory.
No desired position trajectory is constructed; the optimizer instead
compares two accelerations evaluated at the predicted interaction state:
the acceleration the constrained robot will produce and the desired
behavioral acceleration the generator specifies. Position and velocity
enter as generator states and as safety-constrained quantities, never as
tracked references in their own right.

The optimization variable is the robot command sequence $U_k$, not the
generator's parameters. The generator parameters $\theta$ are fixed
during each experiment. Switching from impedance to admittance changes
$C_\theta,G_\theta$, not the QP constraints or the command variable ---
stated formally as Theorem~\ref{thm:invariance} (Section~\ref{sec:properties}).

\section{Properties of the Affine Runtime Instance}
\label{sec:properties}

The architecture does not depend on the following QP property; rather,
the property shows how cleanly the separation is realized for the
present affine subclass. This section establishes template invariance,
convexity, exact realization under idealized conditions, and one-step
constraint inheritance, while Proposition~\ref{prop:inheritance} states
explicitly what the runtime does not guarantee.

\begin{theorem}[Affine-Generator Template-Invariance Property]
\label{thm:invariance}
Fix the plant $(A,B,m_r)$, horizon $N$, weights
$(W,\lambda_u,\lambda_{\Delta u})$, and bounds
$(u_{\max},\dot u_{\max},p_{\max},v_{\max})$. There is a single map
$(C,G)\mapsto Q(C,G)$ --- a quadratic-program template depending only on
the plant, horizon, weights, and bounds above, never on $\theta$ ---
such that, for every generator $\theta$ with affine law
$(C_\theta,G_\theta)$, the realization QP of Section~\ref{sec:qp} is
exactly $Q(C_\theta,G_\theta)$. Consequently, for any two generators
$\theta_1,\theta_2$, the instances $Q(C_{\theta_1},G_{\theta_1})$ and
$Q(C_{\theta_2},G_{\theta_2})$ share the same decision variable $U_k$
and the same feasible set, and differ only in the numerical objective
coefficients contributed by $(C_\theta,G_\theta)$, not in the
decision-variable dimensions, constraint set, or objective template.
Realizing a different affine generator therefore changes only the
supplied $(C_\theta,G_\theta)$. A nonlinear generator can make
$a_{i|k}^{\mathrm{id}}$ nonlinear in $U_k$, so the residual cost is no
longer guaranteed convex quadratic and the template $Q$ does not apply
without linearization.
\end{theorem}

\begin{proof}
Every constraint in Section~\ref{sec:qp}'s QP --- the dynamics
recursion, the acceleration definition, and the four box constraints ---
is written in terms of $A$, $B$, $m_r$, $u_{\max}$, $\dot u_{\max}$,
$p_{\max}$, $v_{\max}$, $U_k$, and the resulting state/acceleration
sequence; none references $\theta$, $C_\theta$, or $G_\theta$. The
feasible set is therefore a fixed polyhedron, identical for every
generator. The only appearance of $(C_\theta,G_\theta)$ anywhere in the
problem is the term
$a_{i|k}^{\mathrm{id}}=C_\theta x_{i|k}+G_\theta\hat F_{h,i|k}$ inside the
primary cost, entering exactly once per horizon step as an affine
substitution into $r_{i|k}=a_{i|k}-a_{i|k}^{\mathrm{id}}$. Because
$x_{i|k}$ is itself independent of $(C_\theta,G_\theta)$ and affine in
$U_k$ (established in Proposition~\ref{prop:convexity}'s proof below),
$a_{i|k}^{\mathrm{id}}$ is affine in $U_k$ with
$(C_\theta,G_\theta)$-dependent coefficients only, so $\|r_{i|k}\|_W^2$
is a convex quadratic in $U_k$ whose Hessian and gradient contributions
are determined entirely by $(C_\theta,G_\theta)$; the secondary terms
$\lambda_u\|u_{i|k}\|_2^2$ and
$\lambda_{\Delta u}\|u_{i|k}-u_{i-1|k}\|_2^2$ do not involve
$(C_\theta,G_\theta)$ at all. Collecting these pieces defines $Q(C,G)$
as a function of $(C,G)$ alone --- the plant, horizon, weights, and
bounds enter $Q$ as fixed parameters, not arguments --- and by
construction $Q(C_\theta,G_\theta)$ is exactly the QP
Section~\ref{sec:qp} poses for generator $\theta$.
\end{proof}

This is why Section~\ref{sec:generators}'s two implemented generators
are checkable rather than aspirational: each supplies a different
$(C_\theta,G_\theta)$ to the same template $Q$. Without linearization, a
nonlinear or stateful $f_\theta$ does not factor through a fixed $(C,G)$
pair and lies outside the theorem.

\begin{proposition}[Convexity for affine generators]
\label{prop:convexity}
If $W\succeq0$, $\lambda_u\ge0$, and $\lambda_{\Delta u}\ge0$, then the
finite-horizon problem in Section~\ref{sec:qp} is a convex quadratic
program. It is strictly convex if the assembled Hessian is positive
definite, including the common case $\lambda_u>0$.
\end{proposition}

\begin{proof}
The lifted state $x_{i|k}$ is affine in $U_k$ by the recursion
$x_{i+1|k}=Ax_{i|k}+B(u_{i|k}+\hat F_{h,i|k})$ unrolled from $x_{0|k}$.
Hence $a_{i|k}=(u_{i|k}+\hat F_{h,i|k})/m_r$ and
$a_{i|k}^{\mathrm{id}}=C_\theta x_{i|k}+G_\theta\hat F_{h,i|k}$ are
affine in $U_k$, so $r_{i|k}=a_{i|k}-a_{i|k}^{\mathrm{id}}$ is affine in
$U_k$ and $\|r_{i|k}\|_W^2$ is convex quadratic since $W\succeq0$. The
regularization terms $\lambda_u\|u_{i|k}\|_2^2$ and
$\lambda_{\Delta u}\|u_{i|k}-u_{i-1|k}\|_2^2$ are convex quadratic in
$U_k$ for $\lambda_u,\lambda_{\Delta u}\ge0$, and a sum of convex
functions is convex. Every listed constraint bounds an affine function
of $U_k$ in absolute value, so the feasible set is an intersection of
half-spaces, hence a convex polyhedron.
\end{proof}

\begin{proposition}[Exact unconstrained realization]
\label{prop:exact}
Assume additionally that $W\succ0$ (strictly positive definite --- a
stronger hypothesis than Proposition~\ref{prop:convexity}'s
$W\succeq0$). Suppose an input sequence exists for which $r_{i|k}=0$
over the horizon and no constraint is active. If the realization
residual is optimized lexicographically before secondary effort
objectives---or equivalently the secondary weights are zero---then every
optimal primary solution realizes the reference dynamics exactly.
\end{proposition}

\begin{proof}
The primary objective $\sum_i\|r_{i|k}\|_W^2$ is nonnegative and attains
the value zero at the assumed feasible sequence, so its minimum over the
horizon is zero. Under lexicographic priority (or zero secondary
weights), any optimal solution must attain this minimum, so
$\sum_i\|r_{i|k}\|_W^2=0$, and since each term is nonnegative,
$\|r_{i|k}\|_W^2=0$ for every $i$. Because $W\succ0$, a quadratic form
$\|r\|_W^2$ vanishes only at $r=0$; if $W$ were only positive
semi-definite, this step would give $r_{i|k}\in\ker W$ rather than
$r_{i|k}=0$, which is why the stronger hypothesis is needed here but not
in Proposition~\ref{prop:convexity}.
\end{proof}

With finite (not zero) secondary weights, predictive realization's
measured residual is not exactly zero even during unconstrained
intervals (Table~\ref{tab:planar}), consistent with
Proposition~\ref{prop:exact}'s zero-weight idealization. The reactive
comparator instead solves a different, unconstrained algebraic
inversion of the generator's instantaneous law, so its residual sits at
numerical-solver tolerance (Table~\ref{tab:planar}) --- a check that the
underlying generator and dynamics model are implemented correctly, not
itself an instance of Proposition~\ref{prop:exact}.

\begin{proposition}[One-step constraint inheritance]
\label{prop:inheritance}
Assume the model and current human force are exact over the executed
sample and the QP is feasible. Then $u_{0|k}^\star$, $p_{1|k}$, and
$v_{1|k}$ satisfy their corresponding QP bounds. Re-solving at every
sample yields constraint satisfaction by induction while feasibility is
retained.
\end{proposition}

This proposition is not a recursive-feasibility theorem. A deployable
version requires a terminal invariant set, soft constraints with a
quantified fallback, or a backup safe controller. Sudden force changes
inside one sample and model error must also be covered by robust
constraint tightening.

\section{Planar Architecture Validation}
\label{sec:planar}

Section~\ref{sec:properties} establishes what the QP guarantees in
principle --- convexity, exact realization when unconstrained, one-step
constraint satisfaction, and what it does not guarantee (recursive
feasibility). This section tests those properties empirically on the
planar plant of Section~\ref{sec:interface}, comparing predictive
realization against the reactive comparator under identical command and
rate limits.

\subsection{Reproducible Setup}

The simulation uses a 2.5~kg planar point mass, $\Delta t=0.02$~s, and a
20-step (0.4~s) horizon. A smooth 12~N force is applied along $+y$ from
1~s to 3~s. Limits are
\begin{equation}
|u_j|\le18\text{ N},\qquad |v_j|\le0.22\text{ m/s},\qquad
|p_j|\le0.10\text{ m},
\end{equation}
with $|\dot u_j|\le180$~N/s. The realization, force, and force-rate
weights are $1$, $2\times10^{-4}$, and $10^{-3}$, respectively.

The impedance generator (Section~\ref{sec:generators}) uses
\begin{equation}
M_d=2.0\text{ kg},\quad D_d=18\text{ Ns/m},\quad K_d=45\text{ N/m}.
\end{equation}
The admittance generator (Section~\ref{sec:generators}) uses
$T_a=0.25$~s and $Y=0.025$~m/(Ns).

The comparator computes the instantaneous reference acceleration,
converts it to robot force, and applies the same 18~N force and
180~N/s slew limits. It has no prediction of position or speed
constraints. This comparator is intentionally minimal: it isolates what
horizon-wide state constraints change; it is not presented as a
survey-complete baseline.

\subsection{Results}
\label{sec:planar-results}

\begin{figure}[t]
\centering
\includegraphics[width=0.95\linewidth]{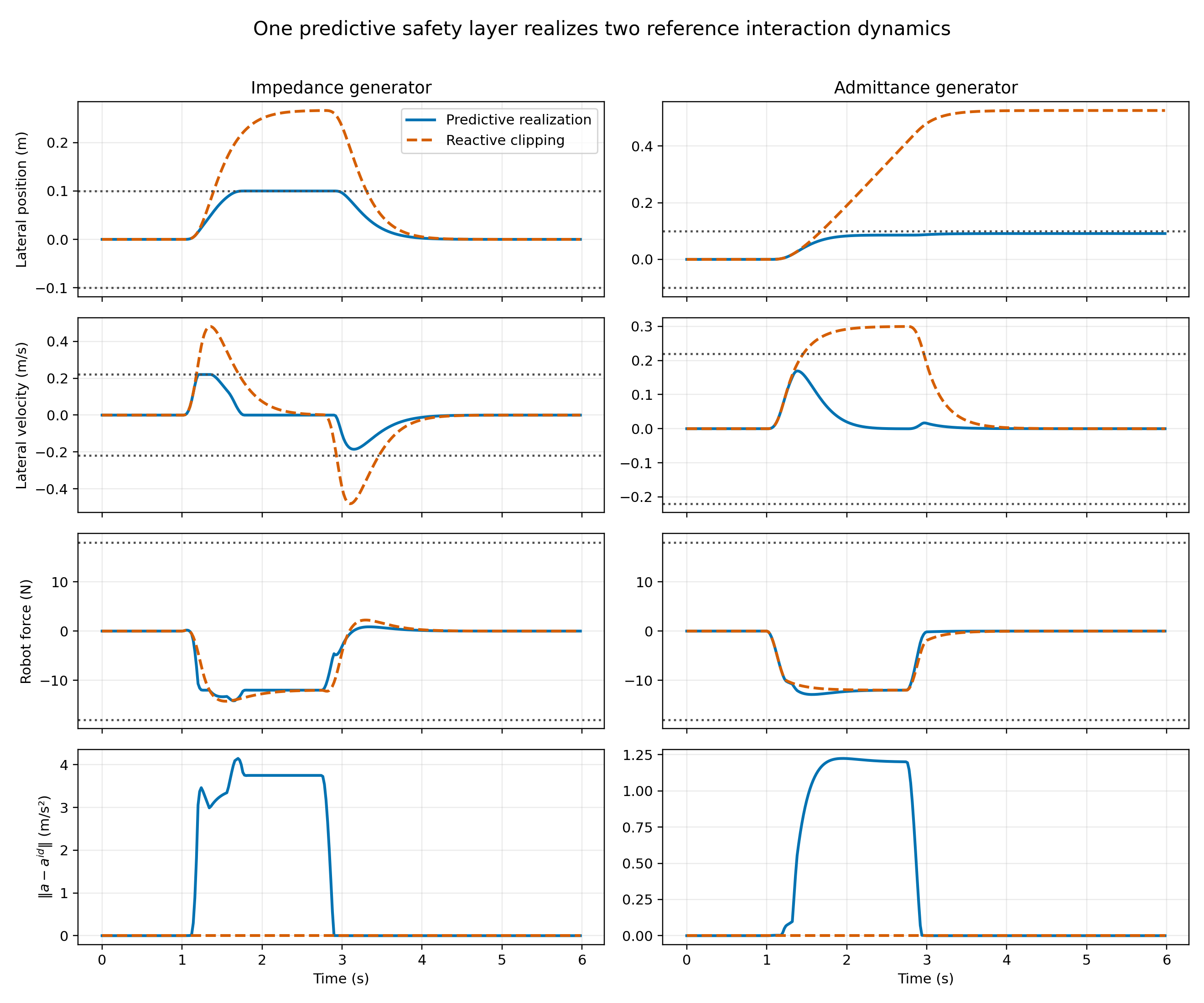}
\caption{Planar validation: impedance (left) and admittance (right).
Dotted lines are state limits. Reactive clipping tracks the behavior
but ignores the limits; predictive realization departs from it near the
bound (final row).}
\label{fig:planar}
\end{figure}

\begin{table*}[t]
\centering
\caption{Planar validation.}
\label{tab:planar}
\small
\begin{tabular}{@{}llrrrrl@{}}
\toprule
Generator & Controller & RMSE$^{*}$ & Peak $|p_j|$ & Peak spd. & Peak $|u_j|$ & Violation \\
& & (m/s$^2$) & (m) & (m/s) & (N) & \\
\midrule
Impedance & Predictive & 1.351 & 0.100002 & 0.220 & 14.16 & $2\times10^{-6}$~m \\
Impedance & Reactive & $<10^{-6}$ & 0.2663 & 0.480 & 14.24 & 0.1663~m, 0.260~m/s \\
Admittance & Predictive & 0.405 & 0.0912 & 0.169 & 12.88 & none \\
Admittance & Reactive & $<10^{-6}$ & 0.5250 & 0.300 & 12.00 & 0.4250~m, 0.080~m/s \\
\bottomrule
\end{tabular}

\vspace{2pt}
\raggedright\footnotesize
$^{*}$Realization RMSE, componentwise. State-limit violation reports
the amount by which the 0.10~m~/~0.22~m/s workspace/speed bound is
exceeded. Realization RMSE for the reactive comparator is at
numerical-solver tolerance because it directly inverts the generator's
instantaneous law; predictive realization's nonzero residual reflects
finite secondary weights and constraint-driven deviation, aggregated
over the whole run (Proposition~\ref{prop:exact},
Section~\ref{sec:properties}). As in Table~\ref{tab:fr3}, RMSE is
\emph{component-wise}: all residual components and time samples are
pooled into one mean before the square root.
\end{table*}

\begin{table}[t]
\centering
\caption{Planar residual decomposition (m/s$^2$).}
\label{tab:planar-decomposition}
\small
\begin{tabular}{@{}lrrrr@{}}
\toprule
Generator & Total & $r^{\rm reg}$ & $r^{\rm con}$ & Closure \\
\midrule
Impedance  & 1.351 & 0.0113 & 1.339 & $<9\!\times\!10^{-16}$ \\
Admittance & 0.405 & 0.00261 & 0.403 & $<3\!\times\!10^{-16}$ \\
\bottomrule
\end{tabular}

\vspace{2pt}
\raggedright\footnotesize
The first three numerical columns are component-wise RMSE and need not
add as scalars; (\ref{eq:residual-decomposition}) closes samplewise.
``Closure'' is its maximum absolute component error. The planar model is
exact, so $r^{\rm mod}=0$.
\end{table}

The impedance reference has a static displacement
$F_h/K_d=12/45=0.267$~m, matching the 0.266~m simulated peak after the
smooth force ramp. Because this behavior is incompatible with the
0.10~m workspace, the predictive controller increases the realization
residual while the bound is active. After release, it returns to the
impedance equilibrium at the origin.

The admittance reference requests approximately 0.30~m/s under 12~N and
has no restoring spring. Reactive realization therefore accumulates
0.525~m displacement and retains it after release. The constrained
controller begins departing from the requested velocity dynamics before
reaching the workspace boundary and settles at 0.091~m. The generator
remains unchanged. Table~\ref{tab:planar-decomposition} shows that the
observed modification is predominantly constraint intervention rather
than secondary regularization; the stronger ``entirely
constraint-induced'' wording would be false because the regularization
term is small but nonzero.

\subsection{Online Generator Switching}
\label{sec:switching}

The preceding two subsections validate the generator interface by
running two separate simulations, one per generator --- informative, but
not yet a demonstration that a single running controller accepts a
different generator without redesign. This section tests that directly.
One \texttt{InteractionDynamicsMPC} instance is constructed once, with
the impedance generator; a small constant lateral force (1~N, well
inside every bound) is applied for the entire 6~s run, with no release.
At $t=2$~s and $t=4$~s, exactly one attribute is reassigned on the live
controller --- \texttt{controller.generator} --- switching impedance
$\rightarrow$ admittance $\rightarrow$ impedance; nothing else (the QP
weights, constraints, decision variable, or the
\texttt{previous\_command} state carried between solves) is touched or
reconstructed.

\begin{figure}[t]
\centering
\includegraphics[width=0.9\linewidth]{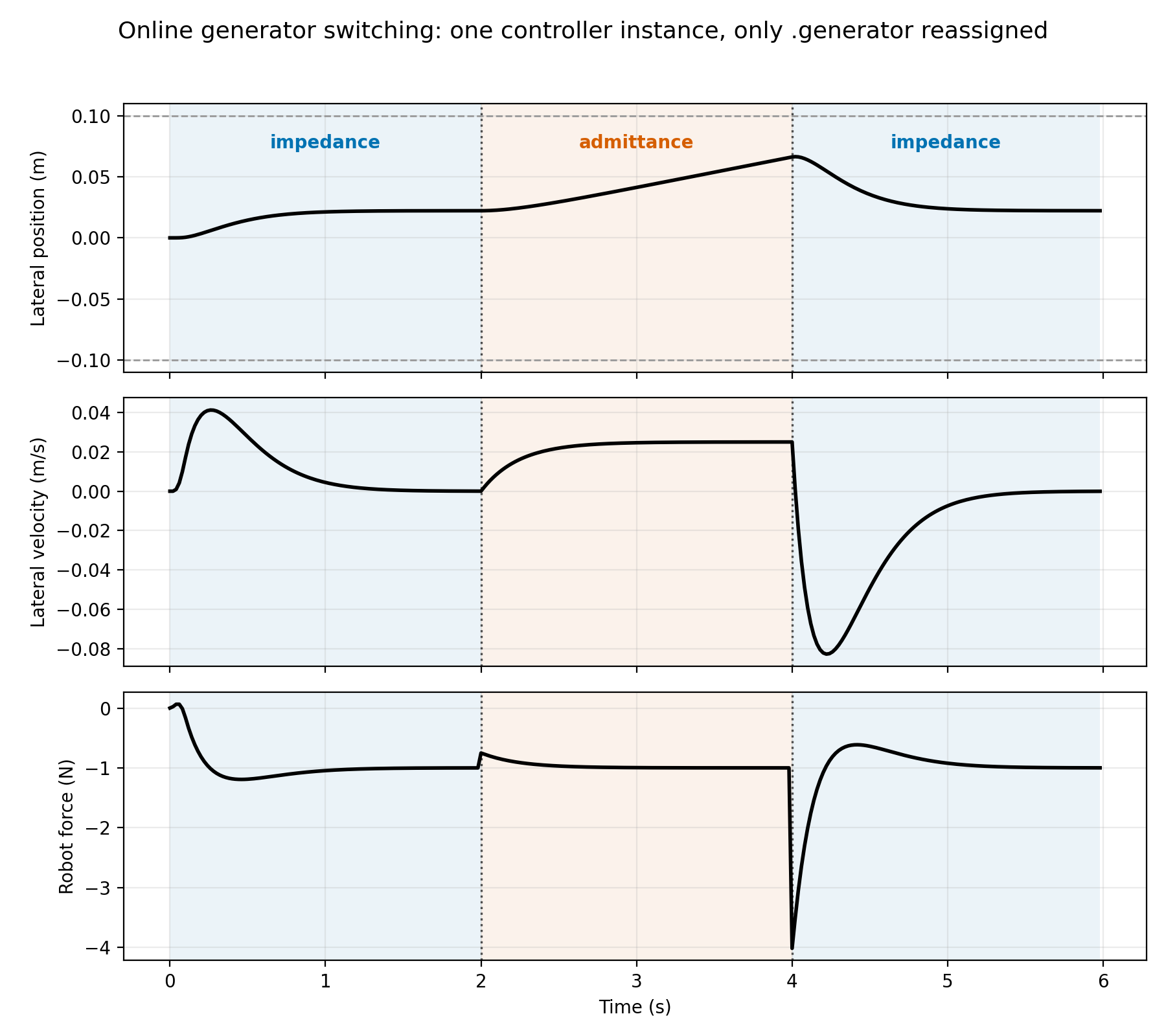}
\caption{One controller instance; only \texttt{.generator} is
reassigned at $t=2$~s and $t=4$~s. Position converges to the impedance
equilibrium in both impedance segments, drifts under admittance in
between, and re-converges after re-entry from a different state.}
\label{fig:switching}
\end{figure}

Both switches are rate-bounded: the largest command-force jump at a
switch instant is 3.02~N at $t=4$~s, 84\% of the QP's 3.6~N per-tick
limit. This is a material command change, not evidence of mathematical
smoothness; it simply satisfies the same discrete rate constraint that
bounds every other step. Position and speed stay within 0.067~m and
0.083~m of the origin throughout, well inside the 0.10~m~/~0.22~m/s
bounds. The experiment therefore validates the software-level generator
interface and feasibility of an online swap under the existing rate
constraint; it does not establish optimal switching, hybrid stability, or
continuity of the generator law. The second impedance segment
additionally shows re-convergence from the different position and
nonzero velocity left by the admittance segment without special-casing
on re-entry.

\subsection{What the Experiments Establish}

The results support four limited claims:
\begin{enumerate}
\item two qualitatively different affine interaction models use the same
constrained predictive implementation;
\item a constrained/unconstrained counterfactual separates constraint
intervention from the small regularization component of total behavior
error;
\item horizon-wide state constraints prevent violations that command
clipping alone cannot prevent in this deterministic model;
\item a single controller instance accepts a different generator online,
mid-run, by reassigning one attribute, with its command change bounded
by the existing rate constraint.
\end{enumerate}
They do not establish manipulator-level feasibility, coupled
human--robot stability, passivity, robustness to force-estimation delay,
or real-time performance on embedded hardware.

\section{FR3 Runtime Validation}
\label{sec:fr3}

Section~\ref{sec:planar} isolates behavior--realization separation from
kinematic and inverse-dynamics confounds by using a point mass. This
section deploys the same interface on a torque-controlled 7-DOF Franka
FR3 simulated in MuJoCo, replacing the exact double-integrator plant
with the nonlinear, configuration-dependent manipulator dynamics of
Section~\ref{sec:interface}.

\subsection{Architecture}
\label{sec:fr3-architecture}

At each solve, let $J_v(q)\in\mathbb R^{3\times7}$ be the translational
Jacobian and $\Lambda(q)^{-1}=J_vM(q)^{-1}J_v^\top$ the (regularized)
inverse operational-space inertia for the position coordinate.

\noindent\textbf{Full closed-loop law.} Let $R_d$ be the held reference
orientation, $e_R(q)\in\mathbb R^3$ the corresponding orientation error,
$J_w(q)\in\mathbb R^{3\times7}$ the rotational Jacobian, and $q_0$ a
fixed neutral configuration. The commanded joint torque is assembled
from a feedforward term $\tau_{\mathrm{ff}}$, a raw orientation-hold
term $\tau_{\mathrm{orient}}$, a raw null-space centering term
$\tau_{\mathrm{null}}$, and their projection $\tau_{\mathrm{aux}}$
through the null-space projector $\bar N_v$:
\begin{equation}
\begin{aligned}
\tau_{\mathrm{ff}}(q,\dot q)&=C(q,\dot q)\dot q+g(q),\\
\tau_{\mathrm{orient}}(q,\dot q)&=J_w(q)^\top\bigl(-K_{\mathrm{rot}}e_R(q)-D_{\mathrm{rot}}\omega\bigr),\\
\tau_{\mathrm{null}}(q,\dot q)&=-k_{\mathrm{null}}(q-q_0)-d_{\mathrm{null}}\dot q,\\
\bar N_v(q)&=I_7-M(q)^{-1}J_v(q)^\top\Lambda(q)\,J_v(q),\\
\tau_{\mathrm{aux}}(q,\dot q)&=\bar N_v(q)^\top\bigl[\tau_{\mathrm{orient}}(q,\dot q)+\tau_{\mathrm{null}}(q,\dot q)\bigr],\\
\tau_{\mathrm{base}}(q,\dot q)&=\tau_{\mathrm{ff}}(q,\dot q)+\tau_{\mathrm{aux}}(q,\dot q),
\end{aligned}
\end{equation}
and the executed torque is
\begin{equation}
\tau=\tau_{\mathrm{base}}(q,\dot q)+J_v(q)^\top F_{\mathrm{cmd}}.
\end{equation}
Task acceleration obeys $\ddot y=J_v\ddot q+\dot J_v\dot q$. Consequently,
define
\begin{equation}
d_{\mathrm{known}}(q,\dot q) =J_v(q)M(q)^{-1}\tau_{\mathrm{aux}}(q,\dot q)
+\dot J_v(q,\dot q)\dot q,
\end{equation}
which contains both the projected auxiliary-torque contribution and the
task-kinematic term. The residual translational dynamics used by the QP
are
\begin{equation}
\ddot e=\Lambda(q)^{-1}\bigl(F_{\mathrm{cmd}}+F_h\bigr)+d_{\mathrm{known}},
\end{equation}
where the realized acceleration is \emph{positive} in the commanded
force, matching the point-mass convention of Section~\ref{sec:interface}.
$F_{\mathrm{cmd}}\in\mathbb R^3$, the QP's Cartesian correction force, is
supplied by whichever controller is active. The predictive controller
sets $F_{\mathrm{cmd}}=F_{0|k}^\star$, the first element of the
receding-horizon solution (Section~\ref{sec:qp}, condensed later in this
section). The reactive comparator instead sets
\begin{equation}
F_{\mathrm{cmd}}=\operatorname{clip}\Bigl(\Lambda(q)\bigl[a^{\mathrm{id}}(x,F_h)-d_{\mathrm{known}}\bigr]-F_h\Bigr),
\end{equation}
rate- and magnitude-limited to the same bounds as the predictive
controller, with $a^{\mathrm{id}}=C_\theta x+G_\theta F_h$ the
instantaneous desired behavioral acceleration
(Section~\ref{sec:generators}). Both controllers share the identical
$\tau_{\mathrm{base}}$ and torque-assembly step; they differ only in how
$F_{\mathrm{cmd}}$ is chosen. Thus the manipulator realization map
changes, but the generator interface does not.

$\bar N_v(q)$ is a dynamically consistent projector for the translation
task~\cite{ref9}. Projecting both the orientation and null-space
auxiliary torques through $\bar N_v(q)$, rather than leaving either
unprojected, prevents them from leaking into $\ddot e$ and inflating
joint drift and the compensating Cartesian command. Because $\Lambda$
uses Tikhonov regularization, $J_vM^{-1}\bar N_v^\top=0$ holds only
approximately; the remaining leakage and $\dot J_v\dot q$ are therefore
retained in $d_{\mathrm{known}}$.

A six-second gain sweep over all four conditions
(\texttt{sweep\_null\_space\_gains.py}) exposed the remaining slow
drift. With $k_{\mathrm{null}}=10,d_{\mathrm{null}}=2$, configuration
deviation reaches 1.32--1.85~rad and the reactive impedance condition
exceeds a joint limit by about 1.6~N$\cdot$m. Gains 40/8 are the
gentlest tested values that remove this violation and keep deviation
below 0.29~rad; they are used in the benchmark. They also alter reactive
task-space peaks (admittance 0.474~m to 0.107~m; impedance 0.158~m to
0.104~m), whereas predictive peaks stay near the 0.06~m workspace bound.
Gains 100/20 reduce drift below 0.19~rad but suppress admittance
further, including its predictive peak to 0.044~m. The selected gains
therefore control drift but are not claimed to leave task-space behavior
unchanged.

The realization QP freezes $\Lambda(q)^{-1}$, $d_{\mathrm{known}}(q,\dot q)$,
$\tau_{\mathrm{base}}(q,\dot q)$, and $J_v(q)$ at the current solve and
holds them across the horizon. In compact form, its decision vector
contains the Cartesian command sequence and nonnegative state slacks,
\begin{equation}
\begin{aligned}
\min_{F_i,s^p_i,s^v_i} & \sum_{i=0}^{N-1} \Bigl(
\|\ddot e_i-a_i^{\mathrm{id}}\|_W^2+\lambda_F\|F_i\|_2^2\\
&+\lambda_{\Delta F}\|F_i-F_{i-1}\|_2^2
+\rho\bigl((s^p_i)^2+(s^v_i)^2\bigr) \Bigr),
\end{aligned}
\end{equation}
\begin{equation}
\begin{aligned}
\ddot e_i&=\Lambda_k^{-1}(F_i+\hat F_{h,i})+d_{\mathrm{known},k},\\
|F_i|&\le F_{\max},\qquad |F_i-F_{i-1}|\le\Delta F_{\max},\\
|\tau_{\mathrm{base},k}+J_{v,k}^{\top}F_i|&\le\tau_{\max},\\
|e_i|&\le e_{\max}+s^p_i,\\
|\dot e_i|&\le v_{\max}+s^v_i,\qquad s^p_i,s^v_i\ge0.
\end{aligned}
\end{equation}

This per-solve model is a convex QP but only a local predictor of the
nonlinear MuJoCo plant. The hard torque constraint is imposed at every
predicted step, rather than only $i=0$. It guarantees feasibility of the
frozen-model command sequence returned by a successful solve; it is not
a horizon-wide guarantee for the future nonlinear trajectory. Executed
torque is therefore recomputed and monitored at 1~kHz.

The torque bound applies to the \emph{total frozen-model torque}
$\tau_{\mathrm{base}}+J_v^\top F_i$, because the base term can consume
actuator budget before the interaction correction is added. A successful
solve therefore satisfies the frozen-model limit by construction. If the
desired behavioral acceleration requires more torque, the QP returns the
closest feasible command and exposes the trade-off through $r_k$.

The workspace and speed boxes use nonnegative slacks with $\rho=10^8$,
because the frozen-Jacobian model cannot guarantee recursive feasibility
for the nonlinear plant. Nontrivial slack occurs in 110 impedance and
186 admittance solves, with peak position slack of 0.077~mm (impedance)
and 0.027~mm (admittance) and peak speed slack below $10^{-6}$~m/s in
both conditions. This slack is smaller than Table~\ref{tab:fr3}'s
executed peak $|e_z|$ violation (0.0001~m impedance, 0.0001~m
admittance) because the two measure different objects: slack is the
QP's own relaxation of its frozen-model prediction, while the executed
violation is measured against the real nonlinear MuJoCo trajectory,
which the frozen model only approximates. The separate reactive
fallback for solver infeasibility is implemented but is not exercised:
every reported solve is feasible.

The fast control loop recomputes $\tau_{\mathrm{base}}$, the orientation
term, and the null-space term from the current $(q,\dot q)$ at 1~kHz;
the QP is re-solved only every 20~ms (a 50~Hz outer loop), holding
$F_{\mathrm{cmd}}$ between solves. Caching the full joint torque at the
outer-loop rate, rather than only the slow correction term, would
confound any comparison of predictive against reactive control at a
common inner-loop rate.

\subsection{Benchmark Scenario}

Mirroring the planar prototype's own story rather than a separate
benchmark design: the EE holds a fixed nominal pose and orientation
while a sustained human push toward a workspace/speed boundary is
applied, and predictive realization is compared against a command/rate-limited
reactive controller with no predictive workspace handling, under an identical
command/rate box. Neither controller has
oracle knowledge of the scripted force; both receive a zero-order-hold
forecast of the currently measured force, as in Section~\ref{sec:planar}.

The push is a 20~N force along $-z$ with a raised-cosine ramp from
$t=1.0$~s to $1.25$~s, a hold to $t=3.25$~s, and a symmetric ramp-down to
$t=3.5$~s. It is applied physically to the MuJoCo end-effector body, not
only supplied to the controller's model.

The 2~s hold gives the admittance velocity time to approach steady state
($T_a=0.3$~s), which is needed for the reactive comparator to approach
the workspace bound and so yield a predictive-versus-reactive contrast
for admittance. The impedance peak, determined by its equilibrium, is
reached within 0.5~s regardless of hold length.

The workspace/speed bound is
\begin{equation}
|e_z|\le0.06\text{ m},\qquad |v|\le0.35\text{ m/s},
\end{equation}
the QP horizon is 15 steps (0.3~s) at $\Delta t=0.02$~s, and per-joint
torque limits match the FR3 ($\pm87$~N$\cdot$m for joints 1--4,
$\pm12$~N$\cdot$m for joints 5--7).

The impedance generator (Section~\ref{sec:generators}) uses $M_d=2.0$~kg,
$D_d=28$~Ns/m, $K_d=200$~N/m. The admittance generator
(Section~\ref{sec:generators}) uses $T_a=0.3$~s and $Y=0.01$~m/(Ns).
Each condition runs for 6~s of simulated time.

\subsection{Results}
\label{sec:fr3-results}

\begin{figure}[t]
\centering
\includegraphics[width=0.95\linewidth]{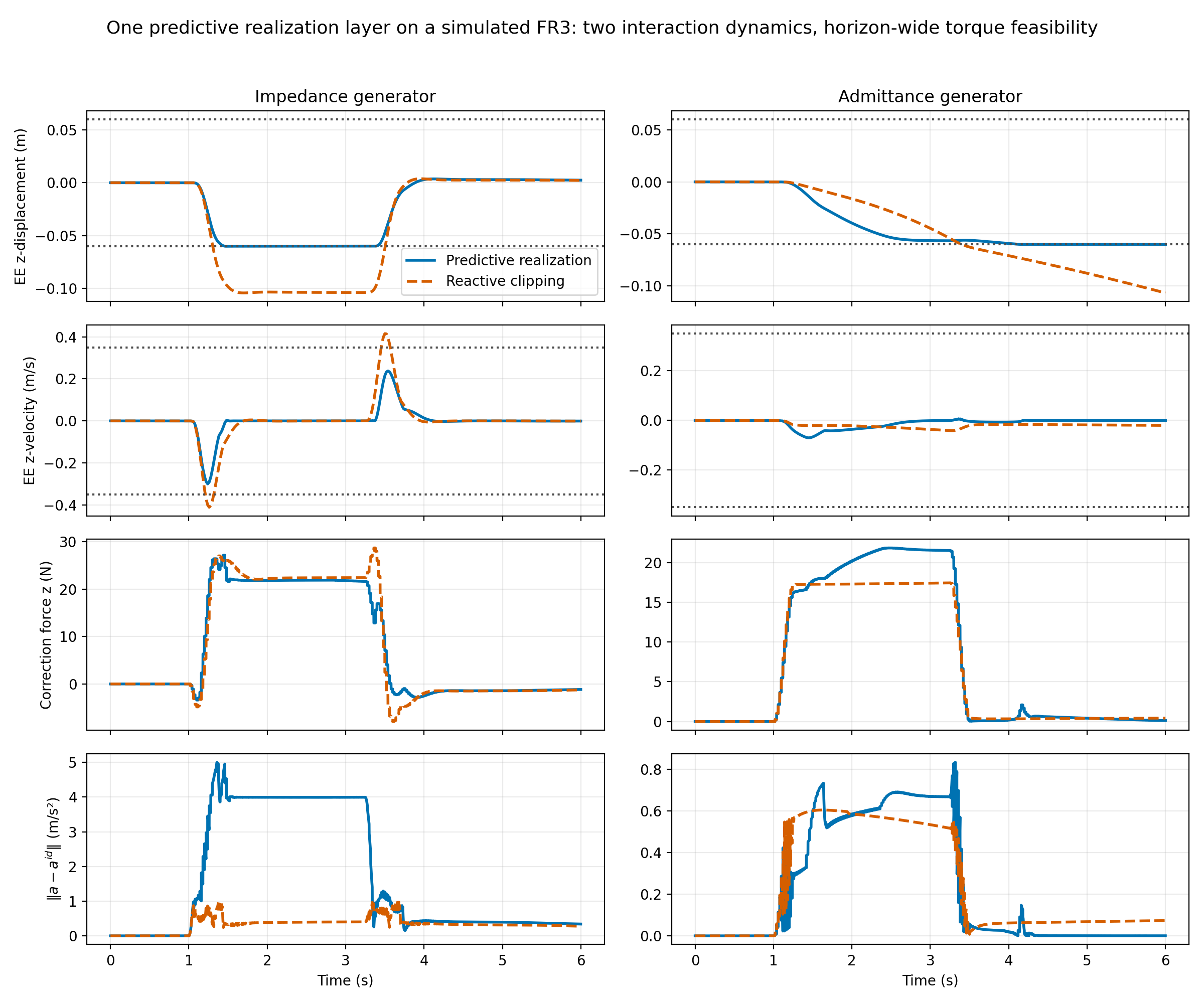}
\caption{FR3 benchmark: impedance (left) and admittance (right),
predictive realization (solid) vs.\ the command/rate-limited reactive
controller without predictive workspace handling (dashed). Dotted
lines are workspace/speed bounds; bottom row is the empirical residual;
torque does not clip in this primary benchmark. Predictive realization
tracks the bound; the reactive controller overshoots
and, for admittance, never recovers.}
\label{fig:fr3}
\end{figure}

\begin{table*}[t]
\centering
\caption{FR3 benchmark.}
\label{tab:fr3}
\small
\begin{tabular}{@{}llrrrrl@{}}
\toprule
Generator & Controller & RMSE$^{*}$ & Peak $|e_z|$ & Peak spd. & Torque util. & Violation \\
& & (m/s$^2$) & (m) & (m/s) & & \\
\midrule
Impedance & Predictive & 1.39 / 1.45 & 0.0601 & 0.300 & 37.0\% & 0.0001~m \\
Impedance & Reactive & 0.223 / 0.062 & 0.1044 & 0.415 & 37.2\% & 0.0444~m, 0.065~m/s \\
Admittance & Predictive & 0.212 / 0.289 & 0.0601 & 0.070 & 34.4\% & 0.0001~m \\
Admittance & Reactive & 0.200 / 0.034 & 0.1066 & 0.041 & 31.7\% & 0.0466~m \\
\bottomrule
\end{tabular}

\vspace{2pt}
\raggedright\footnotesize
$^{*}$Residual RMSE, empirical / predicted. The \emph{predicted}
residual uses the frozen local model of Section~\ref{sec:fr3-architecture}.
The \emph{empirical} residual instead finite-differences the 1~kHz
MuJoCo end-effector velocity and compares that acceleration with the
generator request; it is the appropriate measure of behavior delivered
by the simulated nonlinear plant. Their gap measures local-model and
differentiation error, so the predicted residual must not be interpreted
as an exact plant residual. As in Table~\ref{tab:planar}, both use
component-wise RMSE: all three residual components and all time samples
are pooled into one mean before the square root. This harmonizes the
aggregation convention but does not make cross-plant values direct
performance rankings. Speed is reported component-wise to match the box
constraint. Maximum torque utilization is the largest ratio
$|\tau_j|/\tau_{\max,j}$, which avoids comparing unlike 87~N$\cdot$m and
12~N$\cdot$m joint limits. Predictive solve times are reported in the
text below because the reactive comparator does not solve a QP.
\end{table*}

\begin{table*}[t]
\centering
\caption{FR3 manager-update residual audit (component-wise RMSE, m/s$^2$).}
\label{tab:fr3-decomposition}
\small
\begin{tabular}{@{}lrrrrrr@{}}
\toprule
Generator & Model total & $r^{\rm reg}$ & $r^{\rm con}$ & $r^{\rm mod}$ & Empirical total & Model closure \\
\midrule
Impedance  & 1.448 & 0.285 & 1.205 & 0.199 & 1.395 & $<5\!\times\!10^{-16}$ \\
Admittance & 0.283 & 0.0958 & 0.356 & 0.153 & 0.211 & $<3\!\times\!10^{-16}$ \\
\bottomrule
\end{tabular}

\vspace{2pt}
\raggedright\footnotesize
All columns use the 50~Hz manager-update instants; RMSE columns need not
add because the vector terms can cancel. Model closure is the maximum
absolute component error in $r^{\rm model}=r^{\rm reg}+r^{\rm con}$.
$r^{\rm mod}$ then closes (\ref{eq:residual-decomposition}) against the
finite-difference plant acceleration.
\end{table*}

No condition violates a per-joint torque limit at the executed sample,
no solve is reported infeasible, and the torque constraint never binds
in this primary benchmark: maximum utilization stays below 38\%. This
benchmark therefore evaluates behavior realization at a workspace
boundary, not torque intervention. Section~\ref{sec:torque-activation}
introduces a deliberately derated torque-budget stress case in which the
constraint activates.

An end-to-end unit test also tightens one joint limit until the
horizon-wide constraint activates.

The impedance reference has an unconstrained static displacement
$F_h/K_d=20/200=0.10$~m, which by itself already exceeds the 0.06~m
bound. Reactive clipping tracks toward that equilibrium and peaks at
0.104~m (slightly above the static value, from the same ramp-overshoot
effect as the planar case in Section~\ref{sec:planar-results}),
overshooting the bound by 4.4~cm. Predictive realization instead departs
from the desired behavioral acceleration while the bound is active,
holding displacement to 0.0601~m.

The admittance reference has no equilibrium displacement to compare
against, only a steady-state velocity $Y F_h=0.01\times20=0.2$~m/s while
the force is held, with no restoring term to pull it back after release.
Reactive clipping accumulates displacement through the hold and
continues briefly after force release before the velocity decays (time
constant $T_a=0.3$~s), peaking at 0.1066~m --- a 4.7~cm overshoot ---
and never recovers it, since the generator itself has no
position-restoring term (Section~\ref{sec:generators}) and the reactive
comparator has no lookahead to anticipate the boundary. Predictive
realization begins departing from the requested velocity before the
bound is reached and stays within 0.0601~m.

Table~\ref{tab:fr3-decomposition} prevents the raw FR3 residual from being
misread as constraint intervention alone. Constraint intervention is the
largest component for both generators at manager updates, but finite
regularization and plant/model terms are material, particularly for the
admittance condition where vector cancellation makes the empirical total
smaller than the individual component RMSE values.

The deployed QP uses the exact variable change $z=10^3s$ for each physical
state slack $s$. Constraint coefficients use $z/10^3$, the slack Hessian
weight becomes $\rho/10^6$, and reported values decode $s=z/10^3$; hence
the physical objective and feasible set are unchanged. A unit test verifies
this matrix transformation directly. It reduces the condensed Hessian
condition number from $4.3\times10^9$ to $4.34\times10^3$ for impedance and
from $3.8\times10^9$ to $3.77\times10^3$ for admittance. The deployed 50~Hz
configuration also warm-starts OSQP from the preceding primal/dual solution.

A dedicated timing study (\texttt{run\_fr3\_timing\_study.py}, five full
6~s repetitions and 1500 actual solves per generator/condition) then measures
the complete QP construction and solve. In the deployed warm-started
configuration, impedance has 3.45~ms mean, 6.89~ms 99th-percentile, and
8.67~ms maximum solve time; admittance has 3.68, 7.28, and 9.73~ms,
respectively. Thus all 3000 measured warm-started solves meet the 20~ms
manager period on the development machine. The matched cold-start ablation
has means of 4.24 and 4.02~ms, 99th percentiles of 7.25 and 8.33~ms, and
maxima of 17.73 and 14.92~ms; all 3000 cold solves also meet 20~ms in the
saved scaled-formulation study. Warm-starting is therefore retained for
additional margin, rather than being credited as the sole deadline fix.

These wall-clock measurements are evidence for the tested machine, not a hard
real-time or hardware-independent guarantee. They include QP construction
and OSQP but exclude the post-solve unconstrained counterfactual used only for
the residual audit. The simulator invokes the solve synchronously and does
not implement asynchronous late-result discard; deployment on a different
processor must re-run the timing test and, if overruns occur, hold the last
applied command and reject stale solutions explicitly.

\subsection{Torque-Active Runtime Intervention}
\label{sec:torque-activation}

The primary benchmark uses the nominal FR3 torque limits and does not
activate them. To exercise runtime intervention without changing the
behavior layer, we repeat the impedance condition with joint 4's
available budget deliberately derated from its nominal 87~N$\cdot$m to
31.5~N$\cdot$m. This is an artificial actuator-budget stress test, not a
claim about the FR3's physical rating. The 20~N force profile, impedance
behavior, 0.3~s prediction horizon, QP weights, workspace/speed bounds,
and 50~Hz solve rate remain unchanged.

Two predictive runtimes are compared. The proposed runtime constrains
total frozen-model torque at all 15 predicted steps. The ablation
constrains the same torque only at $i=0$; its QP dimensions, behavior
objective, state constraints, and command/rate bounds are otherwise
identical.

\begin{figure}[t]
\centering
\includegraphics[width=0.95\linewidth]{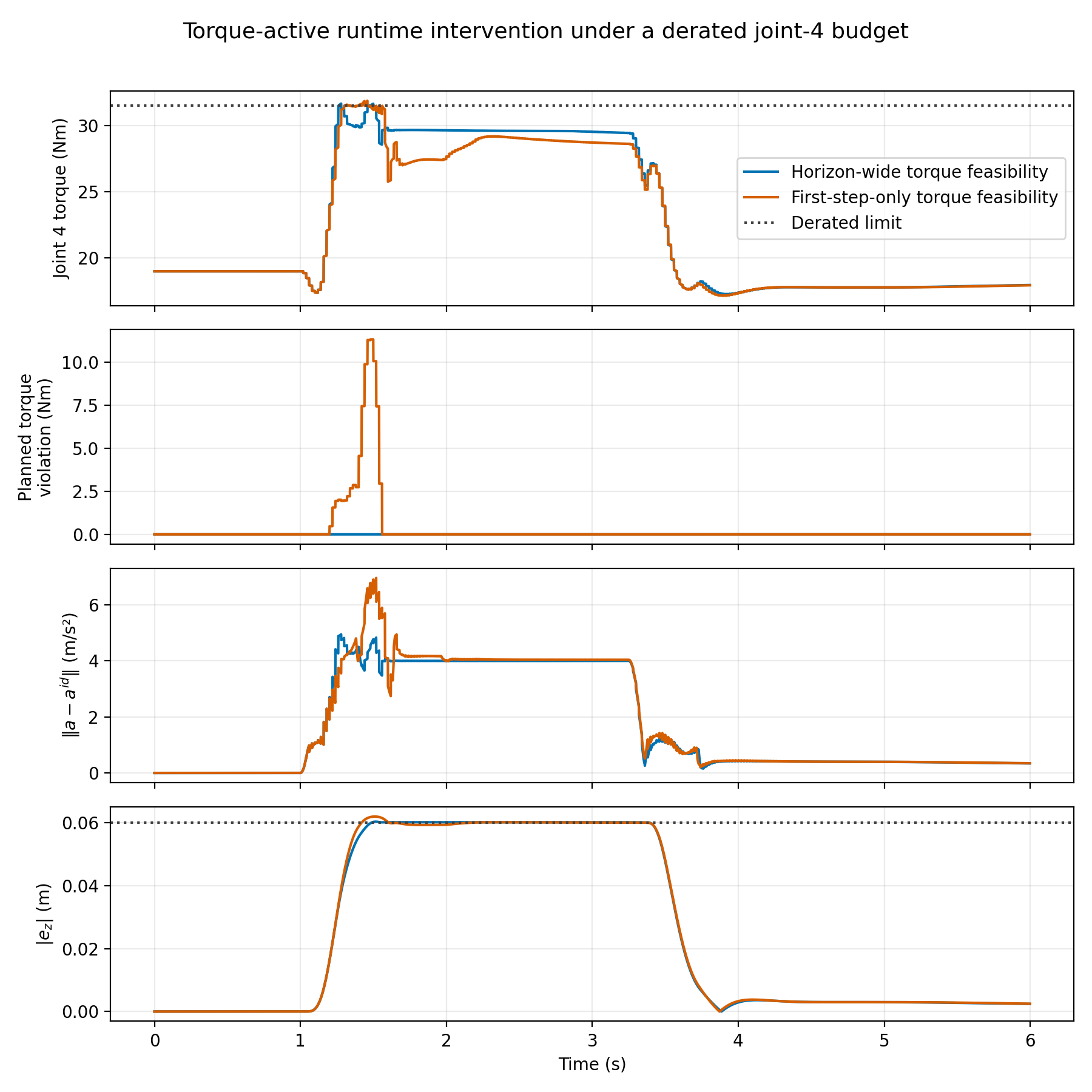}
\caption{Torque-active intervention under a derated joint-4 budget.
Horizon-wide enforcement (blue) keeps the plan feasible; the
first-step-only ablation (orange) satisfies only $i=0$ and plans an
infeasible future during the ramp. Dotted lines are the derated budget
and workspace bound.}
\label{fig:torque}
\end{figure}

\begin{table*}[t]
\centering
\caption{Torque-activation ablation.}
\label{tab:torque}
\small
\begin{tabular}{@{}lrrrr@{}}
\toprule
Enforcement & Planned & Executed & RMSE$^{*}$ & Peak \\
& viol. (N$\cdot$m) & viol. (N$\cdot$m) & (m/s$^2$) & $|e_z|$ (m) \\
\midrule
All 15 steps & 0.0002 & 0.161 & 1.398 & 0.0603 \\
First step only & 11.329 & 0.380 & 1.466 & 0.0619 \\
\bottomrule
\end{tabular}

\vspace{2pt}
\raggedright\footnotesize
$^{*}$Empirical residual RMSE. The planned violation is evaluated
against the frozen model used by each QP. The all-step value is
numerical solver tolerance, whereas the first-step-only runtime plans
future torque up to 11.329~N$\cdot$m beyond the derated budget.
Executed-sample torque is recomputed from the nonlinear MuJoCo state at
1~kHz; its smaller nonzero excess in both cases quantifies the
local-model gap rather than a claimed hard nonlinear-plant guarantee.
Horizon-wide enforcement reduces, but does not eliminate, that mismatch.
The residual uses the component-wise convention of
Tables~\ref{tab:planar} and~\ref{tab:fr3}.
\end{table*}

This experiment is not offered as another controller comparison. It
tests the architectural responsibility assigned to the realization
runtime: when a fixed behavior specification conflicts with an actuator
budget, the runtime changes the command, exposes the resulting
behavioral deviation, and avoids relying on an actuator-infeasible
internal plan. The behavior layer itself is unchanged.

\subsection{Scope and What Remains}
\label{sec:scope}

This is a focused architecture validation, not a complete manipulator
evaluation. The implemented behavior layers are impedance and admittance
only; the scenario is a sustained push rather than a sweep over delay,
sensor noise, or mass mismatch; and there is no collision constraint.
The command/rate-limited reactive controller is a mechanism ablation,
not a survey-complete baseline. The online behavior-layer switching demonstration
(Section~\ref{sec:switching}) is planar only; an FR3 counterpart remains
future work.

Human-participant and hardware validation remain future work, as do a
faithful matched implementation of a predictive safety filter,
predictive admittance controller, or MPIC baseline. Hardware timing and an
explicit stale-solution policy must also be revalidated for deployment even
though the warm-started development-machine study meets 20~ms.

\section{Why Behavior--Realization Separation?}
\label{sec:why}

Interaction behavior should specify only desired interaction dynamics;
physical feasibility should be realized independently by the robot.
Section~\ref{sec:intro} argued this from the different variables each side depends on;
Sections~\ref{sec:generators}--\ref{sec:fr3} tested it through one affine
executable instance. This section defends the separation directly,
against three natural alternatives.

One alternative is to optimize behavior and robot commands jointly.
Joint optimization is appropriate when behavior parameters are
themselves task decisions, but separation instead makes a specification
independently testable and replaceable while assigning feasibility to a
stable downstream contract, at the cost that a fixed behavior layer may
be suboptimal compared with a fully coupled design. This is an argument
for a useful boundary, not a universal dominance claim.

A second alternative is direct robot MPC or trajectory tracking, which
chooses commands to minimize a task objective or evaluates success
against a position/velocity reference. The runtime here instead receives
a law evaluated at the predicted interaction state and measures whether
the robot reproduces the requested response to contact, with no
trajectory generated; MPC is only the current mechanism, while the
architectural choice is what the optimizer receives and what discrepancy
it reports.

A third alternative is a reference governor, which modifies a reference
upstream of an already-designed fixed controller. The realization layer
here is instead itself the command optimizer, and its input is a
desired-acceleration law rather than a setpoint for a separate inner
loop. Both separate intent from constraint handling, so separation alone
is not the novelty; the distinction is the interface boundary, online
cross-behavior substitution, and the decomposition of
desired-versus-realized acceleration.

What separation buys, concretely, is controlled substitutability:
Section~\ref{sec:switching} changes the behavior layer while preserving
the realization object and its carried state, and
Section~\ref{sec:torque-activation} changes the actuator budget while
preserving the behavior --- complementary tests of the same boundary, in
which semantics can change without rebuilding feasibility logic, and
feasibility can intervene without rewriting parameters. 

\section{Limitations and Open Theory}
\label{sec:limitations}

The results above establish the architectural distinction on two plants,
not everything a mature behavior-realization framework needs. Five gaps
are worth stating plainly.

First, the implemented behavior class is affine and memoryless; a
stateful or learned model would need augmented prediction and explicit
treatment of model uncertainty.

Second, the realization cost does not itself imply passivity --- a
passive reference generator can be rendered non-passive by
constraint-induced deviation or secondary optimization. A dissipativity
constraint on the realized port variables, paired with an energy tank or
passivity observer, is a natural extension.

Third, hard state constraints can become infeasible after an unpredicted
impulse (Proposition~\ref{prop:inheritance} flags this as a gap, not a
guarantee). The FR3 study handles it operationally with slack-relaxed
workspace/speed boxes and a reactive fallback on solver infeasibility;
only the slack relaxation is exercised, and it stays small in the tested
scenario. This does not establish a terminal invariant set, quantify
soft-constraint risk, or validate the untriggered fallback --- a
deployable system would need all three.

Fourth, the total realization residual has physical units and depends on
the chosen output coordinates. Equation~(\ref{eq:residual-decomposition})
now separates regularization, joint constraint intervention, and model
error, but does not allocate intervention among individual constraint
rows or provide a normalized severity score. The FR3 study exercises translation only;
orientation is held by a fixed PD law, not a generator, so extending the
interface needs a principled or energy-weighted normalization across
coordinates. Raw RMSE across generators with different scales is not a
direct performance ranking, and Table~\ref{tab:fr3}'s empirical and
frozen-model residuals differ materially in several conditions ---
evidence of model-adequacy and estimation gaps, not a violation of
Theorem~\ref{thm:invariance}.

Finally, the architectural novelty must be assessed against the full
literature of Section~\ref{sec:related}, not only its closest
precedents. The strongest eventual claim rests on implemented
behavior-layer substitutability, formal residual guarantees, and matched
experiments, not terminology alone.

\section{Conclusion}

This paper's contribution is architectural, not another interaction
controller: desired-behavior specification and constrained robot
realization are different responsibilities of robot-control software,
made executable through a desired-acceleration interface and one
predictive realization runtime for its memoryless affine subclass. The
planar study swaps impedance for admittance and back while preserving
the running realization object and its command state. The FR3 study
instead changes the actuator budget while preserving the behavior layer:
horizon-wide torque enforcement modifies the command, exposes the
intervention through the constrained/unconstrained audit, and avoids the
infeasible future plan
a first-step-only ablation produces. Together these test both directions
of the software boundary. The evidence is a reproducible proof of
concept, scoped in Sections~\ref{sec:scope} and~\ref{sec:limitations};
future behavior layers may be physics-based, optimization-based, or
learned only once they expose a compatible realization contract.
Interaction behavior should specify only desired interaction dynamics;
physical feasibility should be realized independently by the robot.

\section*{Reproducibility}

All experiments are driven by fixed, deterministic scenarios and are
verified by an accompanying test suite that runs each scripted benchmark
end to end and asserts its headline claims directly (torque-limit
compliance, solver feasibility, workspace-bound margins, and the
horizon-wide-versus-first-step-only comparison of
Section~\ref{sec:torque-activation}). It also asserts samplewise closure
of the residual decomposition on the planar and frozen FR3 models rather
than relying on inspection
of a saved figure. Physical quantities (position, torque, and violation
counts) are bit-for-bit reproducible from the fixed scenarios; solve
times are wall-clock and are only ever bounded, never asserted exactly.
The warm-start/Hessian-conditioning solve-time diagnosis of
Section~\ref{sec:fr3-results} is produced by a dedicated script
(\texttt{run\_fr3\_timing\_study.py}) separate from the main benchmark,
with unit tests checking the exact slack-variable matrix transformation,
confirming that warm-starting converges to the same command as a cold solve,
and verifying that resetting a controller clears its carried warm-start
state. Code, configuration, and saved numerical
artifacts backing every figure and table in this paper are available in
the project repository.

\IEEEtriggeratref{11}

\end{document}